\documentclass[journal]{IEEEtran}

\usepackage{cite}
\usepackage{mathtools}
\usepackage{amssymb}
\usepackage{enumitem}
\usepackage{dsfont}
\usepackage{{todonotes}}
\usepackage{graphicx}
\usepackage{caption,subcaption}

\usepackage{subcaption}
\usepackage{url}

\usepackage{lineno}

\DeclareMathOperator{\relu}{ReLU}
\DeclareMathOperator{\hull}{Hull}
\DeclareMathOperator{\rank}{rank}
\DeclareMathOperator*{\argmin}{arg\,min}

\newcommand{\norm}[1]{\left\Vert #1 \right\Vert}
\newcommand{\ip}[2]{\left\langle #1, #2 \right\rangle}
\newcommand{\bR}{\mathbb{R}}
\newcommand{\cA}{\mathcal{A}}
\newcommand{\cB}{\mathcal{B}}
\newcommand{\cS}{\mathcal{S}}

\newtheorem{theorem}{Theorem}
\newtheorem{lemma}[theorem]{Lemma}

\newtheorem{corollary}[theorem]{Corollary}
\newtheorem{definition}[theorem]{Definition}
\newtheorem{remark}[theorem]{Remark}
\newenvironment{proof}{{\textbf{Proof\ }}}{~\hfill $\square$}
\newenvironment{proof2}{{\textbf{Proof\ }}}{~\hfill }

\usepackage{makecell}

\ifCLASSINFOpdf
\else
\fi

\begin{document}
\title{Singular Values for ReLU Layers\\}

\author{S\"oren Dittmer, Emily J.\ King, Peter Maass}

\markboth{}
{Shell \MakeLowercase{\textit{et al.}}: Singular Values for ReLU Layers}

\maketitle

\begin{abstract}
Despite their prevalence in neural networks we still lack a thorough theoretical characterization of ReLU layers. This paper aims to further our understanding of ReLU layers by studying how the activation function ReLU interacts with the linear component of the layer and what role this interaction plays in the success of the neural network in achieving its intended task. To this end, we introduce two new tools: ReLU singular values of operators and the Gaussian mean width of operators. By presenting on the one hand theoretical justifications, results, and interpretations of these two concepts and on the other hand numerical experiments and results of the ReLU singular values and the Gaussian mean width being applied to trained neural networks, we hope to give a comprehensive, singular-value-centric view of ReLU layers. We find that ReLU singular values and the Gaussian mean width do not only enable theoretical insights, but also provide one with metrics which seem promising for practical applications. In particular, these measures can be used to distinguish correctly and incorrectly classified data as it traverses the network.  We conclude by introducing two tools based on our findings: double-layers and harmonic pruning.

\end{abstract}

\IEEEpeerreviewmaketitle

\section{Introduction}
\label{sec:intro}
\subsection{Motivation and Overview}
Singular values are an indispensable tool in the study of matrices and their applications. Not only are they used in data science, e.g.\ in principal component analysis~\cite{jolliffe2011principal} and low-rank approximations in general~\cite{golub1987generalization}, but they are also used in the computation of the generalized inverse~\cite{penrose1955generalized}, signal processing~\cite{comon1990tracking}, and the analysis and regularization of inverse problems~\cite{louis1990mollifier}, as well as countless other fields. Since the usual singular values are only applicable in the linear setting, they are unfortunately
not directly suitable for the analysis of nonlinear neural networks. This paper addresses this situation in two ways -- namely by defining ReLU singular values and by defining the Gaussian mean width for operators. We start by generalizing the operator norm to the class of nonnegatively homogeneous operators which includes operators of the form
\begin{equation}
	\relu(A\cdot),
\end{equation}
where $A\in\mathbb{R}^{m\times n}$ is a matrix and ReLU is defined component-wise via $\relu(x):=\max(0, x)$~\cite{nair2010rectified}. Then by leveraging certain traits of singular values, we define ReLU singular values $s_k\in\mathbb{R}_{\ge0}$, $k=0,1,\dots,\min\{m,n\}-1$ given by
\begin{equation}
	\label{eq:def_rsv}
	s_k = \min_{\rank L \le k}\max_{x\in\mathcal{B}} ||\relu(Ax)-\relu(Lx)||_2.
\end{equation}
Here $L\in\mathbb{R}^{m\times n}$ is a matrix with an upper bounded rank and $\mathcal{B}$ the unit ball. After giving a simple but useful analytical bound on ReLU singular values we explore their behavior numerically and see that ReLU layers act like functions with lower singular values than the linear component of the layer would seem to indicate.  This realization motivates the new tool of harmonic pruning.  In Section~\ref{sec:bias},  ReLU singular values are further generalized to allow bias in the ReLU layer and be calculated over actual data rather than general mathematical domains.  The data-dependent, biased ReLU singular values of correctly classified and incorrectly classified data sets as they travel through a trained network reveals structural differences in how networks treat such data. We follow up on this phenomenon in Section~\ref{sec:gmw} where we use the Gaussian mean width~\cite{vershynin2015estimation} to explore the effective dimension of layers and derive how this dimensionality is related to the singular values of linear layers and how the Gaussian mean width can be interpreted with regard to ReLU layers. One of the tools we employ in this analysis is the new concept of Gaussian mean width of operators. We also show that the Gaussian mean width reveals big differences in how networks handle correctly and incorrectly classified data. Before we conclude the paper in Section~\ref{sec:conclusion} we discuss in Section~\ref{sec:app} practical applications of our findings, i.e.\ so-called double-layers and with them harmonic pruning.

In a sense this paper provides a study of sparsity, but unlike most papers on sparsity in neural networks, it is concerned with the sparsity of the spectrum.

Finally, we would like to point out that we provide code for calculating the Gaussian mean width~\cite{code_gaussian_mean_width} and bounds on ReLU singular values~\cite{code_relu_singular_values}. 
\subsection{Related Work}
Although there are generalizations of singular values to the nonlinear setting like \cite{fujimoto2004singular}, in which the nonlinear operators are assumed to be differentiable; there do not seem to be generalizations appropriate for ReLU layers. However, there are a multitude of applied papers documenting and leveraging positive effects of so-called linear bottlenecks in neural networks, e.g.~\cite{MobileNetV2,ba2014deep,sainath2013low}, which we will see are closely connected to ReLU singular values and double-layers. Linear bottlenecks in a network are created by two linear layers (i.e., with no activation function) enforcing a low dimensional representation. These layers are mostly used to deal with huge output dimensions to decrease the number of trainable parameters~\cite{sainath2013low, zeiler2013stochastic} but also sometimes to boost performance~\cite{MobileNetV2, ba2014deep}. Note that bottlenecks, albeit nonlinear ones, are also very common in ResNet residual blocks, e.g.~\cite{he2016deep}. Another related paper which is less concerned with bottlenecks and more concerned with a restriction of the parameters of a neural network is the paper~\cite{li2018measuring}, in which the authors demonstrate that one can train a network with virtually no loss in performance when restricting the network's parameter updates to a very low dimensional subspace of the original parameter space. Additionally the paper~\cite{rippel2015spectral} ties our approach also to pooling methods in convolutional neural networks (CNNs). With regard to the Gaussian mean width we would like to point to~\cite{li2018measuring} as a general source and to~\cite{GiSaBr16} as a paper that applies the Gaussian mean width to untrained neural networks with i.i.d.~Gaussian weights. While not directly applicable to the topics discussed here, we would still like to reference paper~\cite{sedghi2018the}, where the authors investigate the singular values of convolutional layers, ignoring the activation function. 

\section{ReLU Singular Values}
\label{sec:relu_svs}
\subsection{Theory -- A Generalization of Singular Values.}
We would like to open this section by reviewing what singular values are and -- more importantly -- what their essential property might be. We start by stating the following common definition of singular values (see, e.g.,~\cite{gohberg1988introduction}). Let $A\in\mathbb{R}^{m\times n}$ be a real matrix, we then define its $k$th singular value $k=0,1,\hdots, \min\{m,n\}-1$ (Note that we start the indexing at $0$.) as
\begin{align}
	\sigma_k(A) :&= \min_{L\in\mathbb{R}^{m\times n}:\rank L \le k} \max_{x\in\mathcal{B}} ||Ax-Lx||_2 \\
	&= \min_{L\in\mathbb{R}^{m\times n}:\rank(L)\le k} ||A-L||_*,
\end{align}
where $\mathcal{B}\subset\mathbb{R}^n$ is the unit ball and $||\cdot||_\ast$ the operator norm.
This can be phrased as:
\begin{center}
\begin{minipage}{.45\textwidth}
 \textbf{The $k$th singular value of a (linear) operator $A:\mathbb{R}^n\to\mathbb{R}^m$ is the minimal operator norm of $A-L$ with regards to $L$, where $L$ is a (linear) operator $\mathbb{R}^n\to\mathbb{R}^m$ of rank $\le k$.}
\end{minipage}
\end{center}
We would like to draw the reader's attention to the point that, apart from its reliance on the notions $\rank$ and operator norm, this definition is fairly general insofar as it is not inextricably linked to the linearity of the operators involved. This is the first of two key observations we will use for our generalization. The second key observation relates to operator norms and is exemplified by the following equalities: Let $A\in\mathbb{R}^{m\times n}$ be a real matrix.  Then we than can write
\begin{equation}
	||A||_\ast = \max_{x\in\cB} ||Ax||_2 = \max_{x\in\mathcal{S}} ||Ax||_2 = \max_{x\in\mathbb{R}^n\setminus\{0\}} \frac{||Ax||_2}{||x||_2},
\end{equation}
where $\mathcal{S}\subset{\mathbb{R}^n}$ is the unit sphere. The equivalence of these formulations suggests that a key aspect of the operator norm is its positive homogeneity (i.e., its covariant behavior with regard to the multiplication by positive factors) and thereby the representativeness of the operator's restriction to $\mathcal{B}$ and even $\mathcal{S}$. Inspired by this second key observation we will now define nonnegatively homogeneous operators.
\begin{definition} 
Let $V$ and $W$ be Banach spaces over $\mathbb{R}$. A (possibly nonlinear) operator $\cA:V\to W$ is called \textbf{nonnegatively homogeneous} if $\cA\alpha v = \alpha \cA v$ for all $\alpha \ge 0$ and $v\in V$. We then denote the space of all nonnegatively homogeneous operators from $V$ to $W$ as $H_+(V,W)$. 
\end{definition}
Obviously every nonnegatively homogeneous operator is positively homogeneous and every linear operator is nonnegatively homogeneous. Most notably for this paper, $\relu(A\cdot)$, where $A$ is a real matrix, is also nonnegatively homogeneous. For an illustration of $\relu(A\mathcal{S})$, where
\[A=\begin{pmatrix} 0 & 1 \\ 1 & -1 \end{pmatrix}\]
see Figure~\ref{fig:relu_layer}.
\begin{figure}
\center
\includegraphics[width=0.45\textwidth, angle=0]{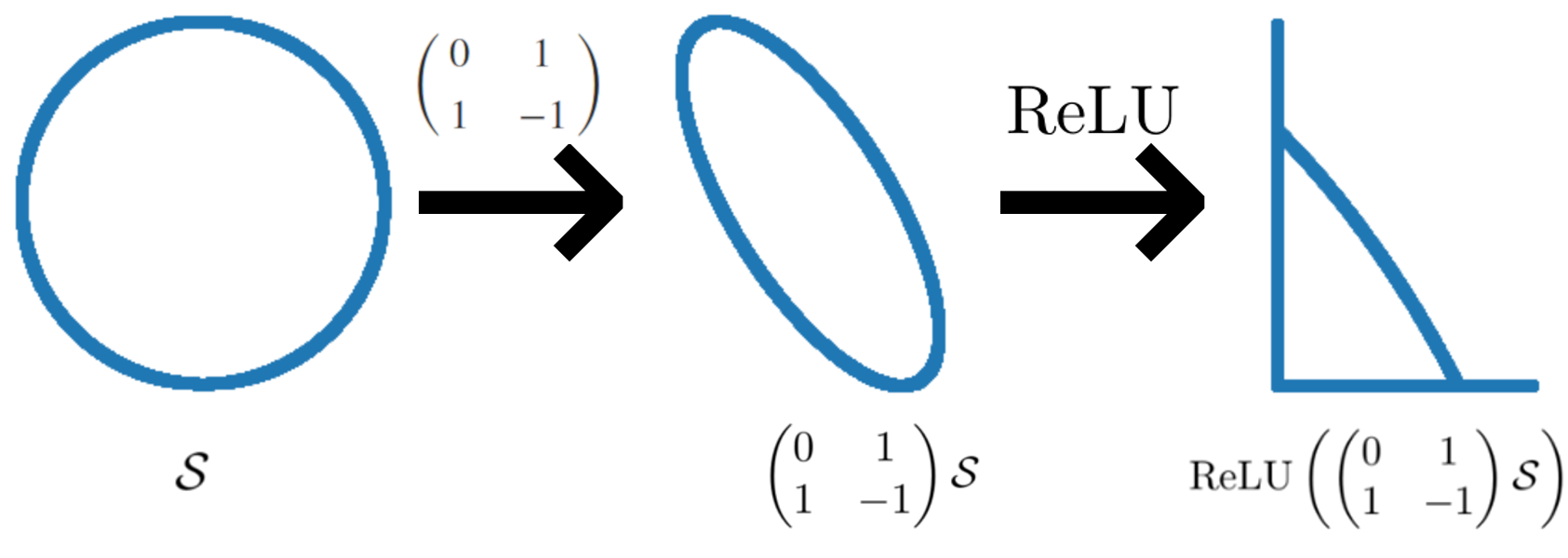}
\caption{Example of how the unit sphere gets mapped by a ReLU layer (without bias). Note that the higher the
dimensions the layer operates in the more points of $\mathcal{S}$ will be mapped to the sparse ``tentacles'' on the right.}
\label{fig:relu_layer}
\end{figure}
Having now defined nonnegatively homogeneous operators we are able to extend the usual definition of operator norms.
\begin{definition}
\label{def:op_norm}
Let $V$ and $W$ be Banach spaces over $\mathbb{R}$. The \textbf{operator norm} of an $\cA\in H_+(V,W)$ is defined as
\begin{equation}
\label{eq:operator_norm}
||\cA||_\ast := \sup_{v\in \mathcal{B}} ||\cA v||_W,
\end{equation}
where $\mathcal{B}$ is the unit ball of $V$.
\end{definition}
In Section~\ref{sec:gmw}, we will see the operator norm of a general function in $H_+(V,W)$ again.  For now, we will focus on the functions in $H_+(\bR^n,\bR^m)$ of the form $x \mapsto \relu(Ax)$ for some matrix $A \in \bR^{m \times n}$ and generalize singular values to such maps.
\begin{definition}
Let $\cA \in H_+(\bR^n,\bR^m)$ be of the form $x \mapsto \relu(Ax)$ for some matrix $A \in \bR^{m \times n}$.  For $k = 0, 1, \hdots, \min\{m,n\}-1$ we define
\begin{align*}
s_k(\cA) &= s_k = \min_{\rank L \le k} \max_{x\in\mathcal{B}} ||\relu(Ax) - \relu(Lx)||_2 \\
&= \min_{\rank L \le k} \norm{\relu(A\cdot) - \relu(L\cdot)}_\ast.
\end{align*}
\end{definition}
Like for linear singular values we have the relation
\begin{equation}
s_{k+1}(\cA) \le s_k(\cA).
\end{equation}
Although the calculation of ReLU singular values seems in general intractable, we may compute an upper bound with the help of the following lemmata.
\begin{lemma}
\label{lemma:upper_bound_a}
For $x, y \in\mathbb{R}^n$ we have 
\begin{equation}
||\relu x - \relu y|| \le ||x - y||.
\end{equation}
\end{lemma}
\begin{proof2}
Since ReLU is applied component-wise, it is sufficient to show that 
\vspace{-2mm}
\[(\relu a - \relu b)^2 \le (a - b)^2,\]
for $a, b \in\mathbb{R}$.
This can is shown by the following analysis of the possible cases.
\begin{tabular}{c | c c }
& $\left(\relu a - \relu b\right)^2$ & $(a-b)^2$ \\
\hline
\makecell{$a\ge0$, $b\ge0$} & $(a-b)^2$ & $(a-b)^2$ \\
\makecell{$a\le0$, $b\le0$} & $0$ & $(a-b)^2$ \\ 
\makecell{$a\le0$, $b\ge0$} & $b^2$ & $(|a|+|b|)^2$ \\ 
\makecell{$a\ge0$, $b\le0$} & $a^2$ & $(|a|+|b|)^2$ 
\end{tabular}\\
$\text{\hspace{83mm}} \square$
\end{proof2}

This lemma allows us to compute the following upper bound.
\begin{lemma}
\label{lemma:upper_bound}
Let $\mathcal{A}\in H_+(\bR^n,\bR^m)$ be given via $\mathcal{A}:=\relu(A\cdot),$
then \[s_k(\mathcal{A})\le\sigma_{k}(A).\]
\end{lemma}
\begin{proof}
For any $B\in\mathbb{R}^{m\times n}$ we have:
\begin{align*}
&\;\;\;\; \max_{x\in\mathcal{B}} ||\relu(Ax) - \relu(Bx)||_2\\
&\le \max_{x\in\mathcal{B}} ||Ax-Bx||_2\\
&\le \max_{x\in\mathcal{B}} ||(A-B)x||_2\\
&= ||A-B||_\ast
\end{align*}
This gives us
\begin{align*}
s_k &= \min_{\rank B \le k} \max_{x\in\mathcal{B}} ||\relu(Ax) - \relu(Bx)||_2\\
&\le \min_{\rank B \le k} ||A-B||_\ast = \sigma_{k}(A)
\end{align*}
\end{proof}

While Lemma~\ref{lemma:upper_bound} gives a direct and intuitive theoretical connection between (linear) singular values and ReLU singular values, we will compute tighter bounds in the following section.
We would like to close this subsection with the following remarks:
\begin{remark}
Although ReLU singular values do not seem to admit a straightforward way to also generalize the singular value decomposition (SVD), they allow for a straightforward definition of a sequence of operators one would associate with the operators given by a truncated SVD in the linear case.
\end{remark}
\begin{remark}
One of the most useful ways to think about a ReLU singular value $s_k(\mathcal{A})$ is to interpret it as the worst case error one has when approximating $\mathcal{A}$ with a rank $k$ approximation. 
\end{remark}
\begin{remark}
The concept of ReLU singular values could be easily extended to arbitrary nonnegatively homogeneous operators, e.g.\ leaky ReLU~\cite{he2015delving} layers. In fact, one can easily see that Lemmas~\ref{lemma:upper_bound_a} and~\ref{lemma:upper_bound} still hold when ReLU is replaced with leaky ReLU.
\end{remark}

\subsection{Numerics -- A Short Numerical Exploration of ReLU Singular Values}\label{sec:numer}
\begin{figure*}[t!]
\begin{subfigure}{0.24\textwidth}
\includegraphics[width=\linewidth]{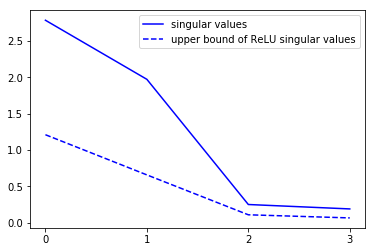}
\caption{Normal($\mu=0$, $\sigma^2=1$)}
\end{subfigure}\hspace*{\fill}
\begin{subfigure}{0.24\textwidth}
\includegraphics[width=\linewidth]{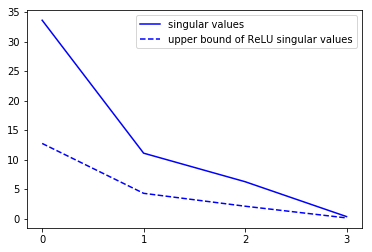}
\caption{Normal($\mu=0, \sigma^2=100$)}
\end{subfigure}\hspace*{\fill}
\begin{subfigure}{0.24\textwidth}
\includegraphics[width=\linewidth]{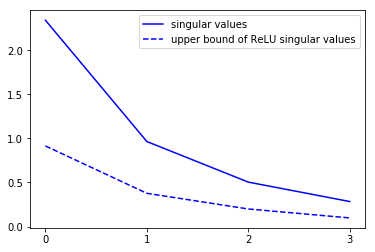}
\caption{Uniform($0, 1$)}
\end{subfigure}\hspace*{\fill}
\begin{subfigure}{0.24\textwidth}
\includegraphics[width=\linewidth]{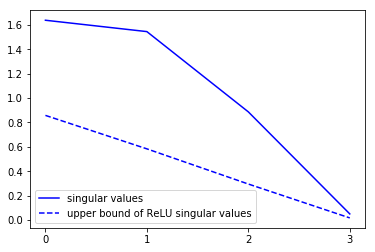}
\caption{Uniform($-1, 1$)}
\end{subfigure}
\caption{Each plot depicts the singular value curve of a matrix $A\in\mathbb{R}^{4\times 4}$ and the corresponding upper bounds of the ReLU singular values of $\relu(A\cdot)$ computing via the method outlines in Section~\ref{sec:numer}. The entries of each respective $A$ are i.i.d.~sampled from different distributions for each $A$ respectively. The sampled distributions are labeled beneath the individual plots.}
\label{fig:singular_values_vs_relu_singular_values}
\vspace{-5mm}
\end{figure*}
We will now describe a simple numerical method for approximating upper bounds of ReLU singular values that is stronger than Lemma~\ref{lemma:upper_bound}. We will then utilize this method to compare these bounds with the singular values of the weight matrices for some random ReLU layers of the form $\relu(A\cdot)$. The approximation is a two step process:
\begin{enumerate}
\item Approximate 
\begin{equation}
W_\ast, M_\ast = \displaystyle\argmin_{\substack{W \in \bR^{m\times k}\\M\in \bR^{k \times n}}} \sum_{x\in X} ||\relu(Ax)-\relu(WMx)||_2^2,
\end{equation}
e.g.\ via the minimization of the function with some kind of stochastic gradient descent (in our case Adam~\cite{kingma2014adam}) for $X$ a sufficiently dense finite subset of $\mathcal{S}$.
\item Calculate an approximate upper bound of $s_k$ as
\begin{equation}
\max_{x\in X}||\relu(Ax)-\relu(W_\ast M_\ast x)||_2,
\end{equation}
this is can be done by simply calculating the norm above for each $x\in X$.
\end{enumerate}
The code of the two-step process can be found here~\cite{code_relu_singular_values}. Its computational complexity is dominated by the first step, since the second is a simple linear search. Therefore the algorithm has effectively the overall complexity of Adam~\cite{kingma2014adam}. In practice this means that it has computational costs similar to training a single layer.
This algorithm yields approximations only, since in practice $X$ has to be finite and is therefore only itself an approximation to $\mathcal{S}$. While this is a minor problem in low dimensions, it poses an increasingly severe problem in higher dimensions due to the curse of dimensionality -- we will deal with this in the next subsection. The algorithm yields an approximation of an upper bound, since it simply ``solves'' Equation~\ref{eq:def_rsv} for the suboptimal guess $L = W_\ast M_\ast$. Figure~\ref{fig:singular_values_vs_relu_singular_values} depicts the comparisons of singular value curves and the corresponding upper bounds of ReLU singular values for some random low dimensional ReLU layers of the form $\relu(A\cdot)$. We define the \textbf{singular value curve} of a real matrix $A$ with an SVD $A=U\Sigma V^T$ as the plot of the (decreasingly) ordered diagonal of $\Sigma$ over the horizontal axis, i.e., the decreasingly ordered singular values over their indices. As one can see the numerical calculations result -- especially for lower indices -- in much lower bounds than the linear singular values give us. That is, the approximations of the ReLU singular values show that the composition of ReLU with a linear function effectively acts like a map with smaller and fewer significant singular values.  This phenomenon is a main motivation of the harmonic pruning method presented in Section~\ref{sec:app}.
\section{Data-Dependent, Affine ReLU Singular Values} \label{sec:bias}
As it probably has not escaped the notice of the more practically minded readers, most ReLU layers have biases, i.e.~are of the form
\begin{equation}
\relu(A\cdot+b).
\end{equation}
This forces us to adapt and rethink our definition of the operator norm in Equation~\ref{eq:operator_norm}, since a direct application of it
\[||\relu(A\cdot+b)||_\ast = \max_{x\in\mathcal{B}} ||\relu(Ax+b)||_2,\]
seems hardly appropriate considering that the operator's application to $\mathcal{B}$ does not have to be representative of the overall behavior of the operator. Three approaches come to mind: a modification of the set we are maximizing over to adapt it to the bias; the removal of the bias; and the incorporation of the bias in the weight matrix, transforming the affine setting back into a linear one. The removal of the bias does not seem appropriate since it clearly has an influence on the operator and since there is not a unique way of transforming the affine setting to a linear one we will concentrate on the set we maximize over. On the one hand it is unfortunately also not clear which modification of $\mathcal{B}$ would be appropriate to capture the behavior of the operator over the whole input space. However, one is usually not interested in how a ReLU layer behaves over its entire mathematical domain since in practice its inputs are samples from a tiny subset of the possible input space. On the other hand there already exists a sufficiently dense representation of that subset in the form of the training data of the network, since otherwise one could not have trained the network. We therefore make the following definitions:
\begin{definition}
Let $A\in\mathbb{R}^{m\times n}$ and $b\in\mathbb{R}^m$. We then generalize the operator norm to the operator $\relu(A\cdot+b)$ over the set $X\subset\mathbb{R}^n$ as
\begin{equation}
||\relu(A\cdot+b)||_{\ast, X} = \max_{x\in X} ||\relu(Ax+b)||_2
\end{equation}
and the \textbf{$k$th ReLU singular value of the operator over a given set $X$} as
\begin{equation}
	s_{k, X} = \min_{\rank L \le k}\max_{x\in X} ||\relu(Ax+b)-\relu(Lx+b)||_2.
\end{equation}
Note that the generalization of the norm is itself not a norm.
\end{definition}
These data-dependent ReLU singular values can be numerically upper bounded in an analogous fashion as in the unbiased case. Figure~\ref{fig:in_correct_relu_singular_values} shows the results of two of our numerical tests where we calculated upper bounds of the data-dependent ReLU singular values for each of the layers of two different, trained neural networks, once where $X$ is a set of a certain size only containing correctly classified data points by the network and once where it is of the same size and only contains incorrectly classified data points. Each of the networks is a multilayer perception (MLP) with three hidden ReLU layers and one classification softmax layer trained via a cross-entropy-with-logits loss function and the Adam optimizer~\cite{kingma2014adam} with Tensorflow~\cite{tensorflow2015-whitepaper} standard settings solving a classification task. The training details for the two networks are displayed in Table~\ref{tab:hyper}
\begin{table}
\center
\begin{tabular}{c | c c }
 & HTRU2 & MNIST \\
\hline
batch size & $4$ & $32$ \\
number of epochs & $8\cdot10^6$ & $5\cdot10^6$ \\
initialization & Glorot~\cite{glorot2010understanding,he2015delving} & Glorot \\
final accuracy on test set & $0.98$ & $0.98$ \\
\end{tabular}
\caption{Training hyperparameters of the networks used in Figure~\ref{fig:in_correct_relu_singular_values} and~\ref{fig:gmw_over_time}.}
\label{tab:hyper}
\vspace{-05mm}
\end{table}
\begin{figure}[t!]
\vspace{-5mm}
\begin{subfigure}[t]{0.24\textwidth}
\includegraphics[width=\linewidth]{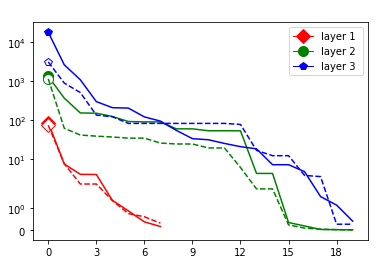}
\caption{The ReLU layers have a width of 20 and solves the binary classification task giving via the dataset HTRU2~\cite{lyon2016fifty}. 
Note that the vertical axis has a log scale for readability.}
\end{subfigure}\hspace*{\fill}
\begin{subfigure}[t]{0.24\textwidth}
\includegraphics[width=\linewidth]{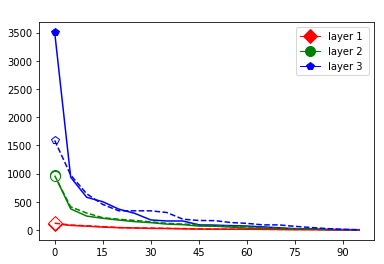}
\caption{The ReLU layers have a width of 100 and the network solves the classical MNIST classification problem~\cite{lecun1998mnist}.} 
\end{subfigure}\hspace*{\fill}

\begin{subfigure}[t]{0.24\textwidth}
\includegraphics[width=\linewidth]{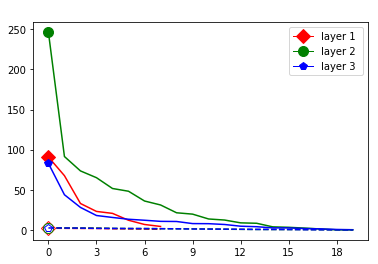}
\caption{The singular value curve of the weight matrices for of the HTRU2 dataset for comparison. The dashed lines are the singular value curves after the initialization, before the training.}
\end{subfigure}\hspace*{\fill}
\begin{subfigure}[t]{0.24\textwidth}
\includegraphics[width=\linewidth]{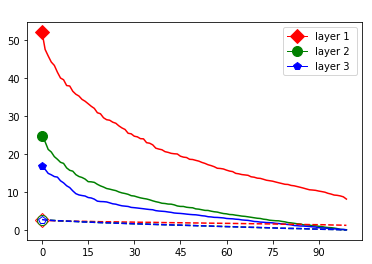}
\caption{The singular value curve of the weight matrices for of the MNIST dataset for comparison. The dashed lines are the singular value curves after the initialization, before the training.}
\end{subfigure}\hspace*{\fill}
\caption{The plots (a) and (b) depict a direct comparison of the numerical upper bounds of the data dependent ReLU singular values of the three hidden ReLU layers of a 3 layer MLP. The three ReLU layers are of the same width and the network solves a classification task. The dashed line represents the data-dependent ReLU singular values for misclassified data by the network and the solid line for correctly classified data. The plots (c) and (d) display the singular value curves of the weight matrices of the layers for comparison.}
\label{fig:in_correct_relu_singular_values}
\vspace{-5mm}
\end{figure}
The plots in Figure~\ref{fig:in_correct_relu_singular_values} display some properties which where fulfilled by the overwhelming majority of our test cases (see Appendix~\ref{appendix}): The bounds of the first layer are smaller than those of the second layer, and the bounds of the second layer are smaller than those of the third (last hidden) layer. Also the greatest bounds of a layer for the correctly classified data tend to be greater than the bounds for the incorrectly classified data. The growth of the singular values over the layers can easily be explained by the fact that the datasets over which the bounds are calculated were already propagated through the previous layers. The relation between the correctly and incorrectly classified datasets is the more interesting characteristic and motivates the next section.

\section{Gaussian Mean Width and ReLU layers}
\label{sec:gmw}
As we saw in the last section ReLU layers seem to handle data points that will be classified incorrectly differently from those that will be classified correctly. To study this difference we briefly want to consider which mechanisms are at play when one applies a layer of the form
\begin{equation}
\relu(A\cdot), \quad A\in\mathbb{R}^{m\times n},
\end{equation}
to a data point $x\in\mathbb{R}^n$. 
Using the SVD of $A(=U\Sigma V^T)$ and an on-$x$-dependent diagonal matrix $D_x$ with $0$s and $1$s on its diagonal to represent ReLU, we can express the application of this layer to $x$ as
\begin{equation}
\label{eq:relu_svs_interaction}
\relu(Ax) = D_xU\Sigma V^Tx.
\end{equation}

This expression shows that the norm of the output depends on how much $x$ is affected by each singular value of $A$, i.e., how much it is correlated with which right singular vector and how much ReLU can decrease the impact of the singular value for a given $x$ by setting entries to $0$. This can be studied, as Lemma~\ref{lemma:gmw} will show, using the Gaussian mean width~\cite{vershynin2015estimation}.

\subsection{Theory -- A Short Review and Expansion.}
We start by stating the definition of the Gaussian and spherical mean width.

\begin{definition}
The \textbf{Gaussian mean width of a set} $K\subset\mathbb{R}^n$~\cite{vershynin2015estimation} is defined as
\begin{equation}
\omega(K):=\mathbb{E}_{g\sim\mathcal{N}(0,\mathds{1}_n)} \sup_{x\in K-K} \left<g, x\right>,
\end{equation}
where $\mathcal{N}(0,\mathds{1}_n)$ denotes an n-dimensional standard Gaussian i.i.d.~vector and $K-K$ is to be read in the Minkowski sense, i.e., $K-K = \left\{x-y: x,y\in K\right\}$. 
We now introduce the new concept of the \textbf{Gaussian mean width of an operator} $\cA:\mathbb{R}^n\to\mathbb{R}^m$ as
\begin{equation}
\Omega(\cA):=\mathbb{E}_{g\sim\mathcal{N}(0,\mathds{1}_m)} \sup_{y\in \cA(\mathcal{B})-\cA(\mathcal{B})} \left<g, y\right>,
\end{equation}
where $\mathcal{B}\subset\mathbb{R}^n$ is the unit ball.
By sampling $g$ from the uniform distribution over the unit sphere instead of from the standard Gaussian, we analogously define $\overline\omega$, following~\cite{vershynin2015estimation} and similarly introduce $\overline\Omega$ as the \textbf{spherical mean width of a set} and \textbf{spherical width of an operator}, respectively.
\end{definition}

An interesting and useful property of the Gaussian mean width is its invariance under convexification, more precisely:
\begin{corollary}~\cite{vershynin2015estimation}
\label{cor:invConv}
For all $K\subset\mathbb{R}^n$ we have
\begin{equation}
\omega(\hull(K)) = \omega(K),
\end{equation}
where $\hull$ denotes the convex hull.
\end{corollary}

Corollary~\ref{cor:invConv} shows that -- like in the definition of the (linear) operator norm -- it does not matter if one defines the
operator's Gaussian mean width via the unit ball or the unit sphere.

For a more rigorous treatment of the basics of the Gaussian mean width and why its square can be
interpreted as the ``effective dimension,'' we refer to~\cite{vershynin2015estimation} but want to state the following
remark based on~\cite{vershynin2015estimation}.
\begin{remark}
One can estimate a vector $x\in K\subset\mathbb{R}^n$ from $m$ random linear observations where $m$ is proportional to $\omega(K)^2$ with the proportionality factor solely depending on the acceptable (absolute) approximation error. Therefore the Gaussian mean width can be seen as a measure of complexity.
\end{remark}

We now analyze the properties of the Gaussian mean width beginning with Lemma~\ref{lemma:gmw_smw}, which explicitly relates the Gaussian and spherical mean width. The result follows from combining arguments from~\cite[Section 3.5.1]{vershynin2015estimation} with standard results concerning independent gamma random variables~\cite{gooch2010encyclopedic}, but we include the proof here for completeness.
\begin{lemma}
\label{lemma:gmw_smw}
Let $K$ be a set in $\mathbb{R}^n$, then
\[\omega(K) = c_n\overline\omega(K),\]
where
\begin{equation}\label{eqn:cn}
c_n=\sqrt{2}\frac{\Gamma(\frac{n+1}{2})}{\Gamma(\frac{n}{2})}
\end{equation}
and $\Gamma$ is the usual extension of the factorial function.
\end{lemma}
\begin{proof}
We start by proving that we can decompose $g\sim\mathcal{N}(0,\mathds{1}_n)$ into 
\[g\equiv \alpha u,\] 
where $u\sim U(\mathcal{S})$, $\alpha\sim\chi(n)$
are respectively the uniform distribution over the unit sphere and the Chi (not $\chi^2$) distribution -- these are orthogonal projections
of $\mathcal{N}(0,\mathds{1}_n)$. Here $\equiv$ indicates the random variables arise from the same distribution.

Let $p_n(u|r) = \frac{\Gamma(\frac{n}{2})}{2\sqrt{\pi}^n} r^{1-n}$ be the probability density function (pdf) of the uniform distribution over the sphere $S(r)$ of radius $r$ centered at $0$ i.e.~$U(S_n(r))$. Also let $\chi_n(r) = \frac{2}{\sqrt{2}^n\Gamma(\frac{n}{2})} r^{n-1}e^{-r^2/2}$ be the pdf of the Chi distribution of degree $n$. Since $rs:(0,\infty)\times S\mapsto rs\in\mathbb{R}^n\setminus\{0\}$ is bijective and
\[p(u,r)=p_n(u|r)\chi_n(r) = \frac{1}{\sqrt{2\pi}^n} e^{-r^2/2}\]
we have $ru\sim\mathcal{N}(0,\mathds{1}_n).$ Due to the argument above we can write:
\begin{align*}
\omega(K) &= \mathbb{E}_{g\sim\mathcal{N}(0,\mathds{1}_n)} \sup_{x\in K-K} \left<g, x\right>\\
&= \mathbb{E}_{u\sim U(\mathcal{S}), \alpha\sim\chi(n)} \sup_{x\in K-K} \left<\alpha u, x\right>\\
&= \mathbb{E}_{u\sim U(\mathcal{S}), \alpha\sim\chi(n)} \alpha \sup_{x\in K-K} \left<u, x\right>\\
&=  \mathbb{E}_{\alpha\sim\chi(n)}\alpha \mathbb{E}_{u\sim U(\mathcal{S})} \sup_{x\in K-K} \left<u, x\right> \\
&=  \sqrt{2}\frac{\Gamma(\frac{n+1}{2})}{\Gamma(\frac{n}{2})} \mathbb{E}_{u\sim U(\mathcal{S})} \sup_{x\in K-K} \left<u, x\right> \\
&=  \sqrt{2}\frac{\Gamma(\frac{n+1}{2})}{\Gamma(\frac{n}{2})} \overline\omega(K)
\end{align*}
\end{proof}

We will now utilize this explicit connection between the Gaussian and spherical mean width in the following lemma.
\begin{lemma}
\label{lemma:gmw}
Let $\mathcal{A}\in H_+(\mathbb{R}^n, \mathbb{R}^m)$ be a general nonnegatively homogeneous operator.  Then we have the following relations:
\begin{equation}
\label{eq:basic_relation_op_and_gmw}
\Omega(\mathcal{A}) = 2c_m\mathbb{E}_{u\sim U(\mathcal{S})}\sup_{x\in\mathcal{B}}\left<u, \mathcal{A}x\right>,
\end{equation}
\begin{equation}\label{eq:basic_relation_op_and_gmw2}
\Omega(\mathcal{A}) \le 2c_m ||\mathcal{A}||_\ast,
\end{equation}
If $A \in \bR^{m \times n}$ with SVD $A=U\Sigma V^T$, we also have
\[
\Omega(A) = 2c_m \mathbb{E}_{u\sim U(\mathcal{S})} ||A^\top u||_2 = 2c_m \mathbb{E}_{u\sim U(\mathcal{S})} ||\Sigma^\top u||_2.
\]
\end{lemma}
\begin{proof} We calculate
\begin{align*}
\Omega(\mathcal{A}) &= c_m\overline\Omega(\mathcal{A}) 
= c_m \mathbb{E}_{u\sim U(\mathcal{S})} \sup_{y\in\mathcal{A}\mathcal{B}-\mathcal{A}\mathcal{B}} \left<u,y\right>\\
&= c_m \mathbb{E}_{u\sim U(\mathcal{S})} \sup_{x,\tilde x\in\mathcal{B}} \left<u,\mathcal{A}x-\mathcal{A}\tilde x\right>\\
&= c_m \left[\mathbb{E}_{u\sim U(\mathcal{S})} \sup_{x\in\mathcal{B}} \left<u,\mathcal{A}x\right> + \mathbb{E}_{u\sim U(\mathcal{S})} \sup_{x\in\mathcal{B}} \left<u,-\mathcal{A}x\right> \right] \\
&= 2c_m \mathbb{E}_{u\sim U(\mathcal{S})} \sup_{x\in\mathcal{B}} \left<u,\mathcal{A}x\right>.
\end{align*}
This proves Equation~\ref{eq:basic_relation_op_and_gmw}. Then Equation~\ref{eq:basic_relation_op_and_gmw2} follows from Cauchy-Schwarz.

We now let $A \in \bR^{m \times n}$ and claim that for all $u \in U(\cS)$
\[
\sup_{x \in \cB} \ip{u}{Ax} = \norm{A^\top u}_2.
\]
For all $x \in \cB$, we may apply Cauchy-Schwarz to conclude
\[
\ip{u}{Ax} =\ip{A^\top u}{x}\leq \norm{A^\top u}_2 \norm{x}_2 \leq \norm{A^\top u}_2.
\]
Further, presuming without loss of generality that $A^\top u \neq 0$, $\frac{A^\top u}{\norm{A^\top u}_2} \in \cB$ and
\[
\ip{A^\top u}{\frac{A^\top u}{\norm{A^\top u}_2}} = \norm{A^\top u}_2,
\]
as desired.  Hence
\begin{align*}
\Omega(A) &= 2c_m \mathbb{E}_{u\sim U(\mathcal{S})} \sup_{x\in\mathcal{B}} \ip{u}{Ax} \\
&= 2c_m \mathbb{E}_{u\sim U(\mathcal{S})}\norm{A^\top u}_2\\
&= 2c_m \mathbb{E}_{u\sim U(\mathcal{S})}\norm{\Sigma^\top u}_2.
%
\end{align*}
\end{proof}

This lemma demonstrates that the Gaussian mean width of an operator is, at least in the linear case, in some sense a way to measure the accumulative effect of the singular values of an operator and that the Gaussian mean width can also be upper bounded via our more general definition of the operator norm. This makes the Gaussian mean width a good tool for further numerical explorations of the effects seen in Figure~\ref{fig:in_correct_relu_singular_values} and discussed at the beginning of this section.

\subsection{Numerics -- A Big Difference Between Correctly and Incorrectly Classified Data}
Before we can utilize the Gaussian mean width in numerical testing, we derive an algorithm for calculating it.
More specifically we want to calculate a good approximation of 
\[\omega(K):=\mathbb{E}_{g\sim\mathcal{N}(0,\mathds{1}_n)} \sup_{x\in K-K} \left<g, x\right>,\]
where $K\subset\mathbb{R}^n$ is finite. As argued in~\cite{vershynin2015estimation}, due to the Gaussian concentration of measure,
\begin{equation}
\label{eq:gmw_estimate}
\sup_{x\in K-K} \left<g, x\right>
\end{equation}
for one $g\sim\mathcal{N}(0,\mathds{1}_n)$ already yields a good estimate for $\omega(K)$. In practice, to make this estimate more stable, we averaged over the results of 100 samples of $g$. Due to Corollary~\ref{cor:invConv} we can replace Formula~\ref{eq:gmw_estimate} by 
\begin{equation}
\label{eq:gmw_estimate_conv}
\sup_{x\in \hull{K}-\hull{K}} \left<g, x\right>.
\end{equation}
As we will see now, this can be solved via a linear program. In what follows, we use the notation $:$ to represent the horizontal concatenation of the vectors/matrices and \rotatebox[origin=c]{90}{$:$} for vertical concatenation. Additionally we will overload the notation for the set $K=\{x_1,\dots,x_{|K|}\}$ to also represent a matrix $K\in\mathbb{R}^{|K|\times n}$ where each row is given by a different element of the set $K$.
This allows us to formulate the solution of Formula~\ref{eq:gmw_estimate_conv} via the constraint optimization problem
\begin{align}
\max_x\left<g, x\right>, \enskip \textrm{s.t.} \enskip  &\begin{pmatrix} 
1 & \cdots & 1 & 0 & \cdots & 0 \\
0 & \cdots & 0 & 1 & \cdots & 1
\end{pmatrix} 
\begin{pmatrix} 
\alpha  \\
\rotatebox[origin=c]{90}{$:$}\\
\beta 
\end{pmatrix}
=
\begin{pmatrix} 
1  \\ 1 
\end{pmatrix},
\label{eq:prob_constraint}\\
&(\alpha:\beta)\ge0 \enskip \textrm{and}\label{eq:positivety_constraint}\\
&(K^T:-K^T)(\alpha:\beta)^T=x,
\end{align}
where $\alpha, \beta \in \mathbb{R}^{|K|}$.
Rewriting yields the algorithm in form of the following linear program:
\begin{equation}
-\min_{(\alpha:\beta)}
\left<
\begin{pmatrix} 
-K\\
\rotatebox[origin=c]{90}{$:$}\\
K 
\end{pmatrix}
g,
\begin{pmatrix} 
\alpha  \\
\rotatebox[origin=c]{90}{$:$}\\
\beta 
\end{pmatrix}
\right> \enskip \textrm{s.t.} \enskip \eqref{eq:prob_constraint}, \eqref{eq:positivety_constraint} \enskip \textrm{hold}.
\end{equation}
For the problems considered in this paper this algorithm allows us to calculate the Gaussian mean width in well under a second using the SciPy~\cite{scipy} method for linear programming.
We will now utilize this algorithm to approximate the Gaussian mean width of a given dataset propagating through the network after each layer; this is equivalent to calculating the Gaussian mean width for the (admittedly biased and therefore not nonnegatively homogeneous) operator given by each layer and again, like in the ReLU singular value case, an appropriate dataset not equal to $\mathcal{B}$.
\begin{figure}[t!]
\begin{subfigure}[t]{0.24\textwidth}
\includegraphics[width=\linewidth]{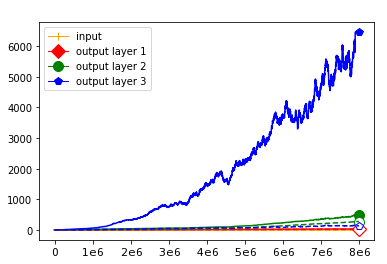}
\caption{The Gaussian mean width over the course of the training of the same network used in Figures~\ref{fig:in_correct_relu_singular_values} (a) and (c) (trained on the dataset HTRU2).}
\end{subfigure}\hspace*{\fill}
\begin{subfigure}[t]{0.24\textwidth}
\includegraphics[width=\linewidth]{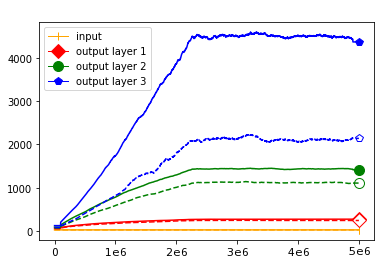}
\caption{The Gaussian mean width over the course of the training of the same network used in Figures~\ref{fig:in_correct_relu_singular_values} (b) and (d) (trained on the dataset MNIST).}
\end{subfigure}\hspace*{\fill}
\caption{The graphs display how the Gaussian mean width (denoted on the vertical axis) changes on the one hand during the propagation through the layers (different colors and shapes) and on the other hand over the course of the training (horizontal axis) of the networks. For comparability we used the same networks as described in Figure~\ref{fig:in_correct_relu_singular_values}. The solid line represents the calculations based on randomly chosen sets of correctly classified data and the dashed line the calculations based on randomly chosen sets of incorrectly classified data of equal size. All graphs where smoothed out for readability. We also checked that the Gaussian mean widths of the input sets (yellow) where on average essentially the same to ensure that we capture differences in processing rather than in data.}
\label{fig:gmw_over_time}
\vspace{-4mm}
\end{figure}
We use the same networks as in Figure~\ref{fig:in_correct_relu_singular_values} and also look at correctly and incorrectly classified data separately. To get a better understanding how these effects come about we monitored how the Gaussian mean width changed over the course of the training. The results are displayed in Figure~\ref{fig:gmw_over_time}; for more see Appendix~\ref{appendix}. There are two very clear effects which both seem to correspond to traits of the ReLU singular values seen in Figure~\ref{fig:in_correct_relu_singular_values}. First, like in the ReLU singular value curves, the Gaussian mean width of each layer seems to get bigger the deeper we are in the network. Second, like the biggest bounds of the ReLU singular values, the Gaussian mean widths of the correctly classified data are bigger than the Gaussian mean widths of the incorrectly classified data, although the effect seems to be clearer for the Gaussian mean width than for the ReLU singular values. This might be explained by the fact that the Gaussian mean width graphs are more accurate due to being smoothed and having a more stable calculation method.  We would also like to point out that in most of our experiments the graph did not ``converge,'' unlike Subfigure~\ref{fig:gmw_over_time}(b), but behaved more like the graph in Subfigure~\ref{fig:gmw_over_time}(a) (even for very small networks after several hours of training).

\section{Applications}
\label{sec:app}
In this section we will first state a hypothesis about the inner functioning of neural networks (or at least MLPs) based on the results of Sections~\ref{sec:gmw} and~\ref{sec:relu_svs}, more specifically the Figures~\ref{fig:in_correct_relu_singular_values} and~\ref{fig:gmw_over_time}. We will then build practical tools based on this hypothesis and evaluate them numerically.
\subsection{A Hypotheses and Summary Based on the Existing Results} 
We start by considering Figure~\ref{fig:in_correct_relu_singular_values}. The figure shows that the singular value curves of the weight matrices of ReLU layers tend to have some singular values dominating the others. Subfigure~\ref{fig:in_correct_relu_singular_values}(c) shows a (in our experiments) typical case, while Subfigure~\ref{fig:in_correct_relu_singular_values}(d) is more of an extreme case in that does not have fast decay. Interestingly this drop off of the singular value curves can be seen in both networks, even more extremely in the ReLU singular value curves, see Subfigures~\ref{fig:in_correct_relu_singular_values}(a) and~\ref{fig:in_correct_relu_singular_values}(b).  The ReLU singular values are in a sense more representative of the networks since they also reflect the effects of the ReLU activation function and the bias. What this means is that low-rank approximations of a layer, low-rank in the sense that weight matrix has low rank, are already very good approximations of the layer, indicating that the layers are ``essentially'' low-rank. This in turn means that only a few of the correlations of a data point with the singular vectors of a layer's weight matrix matter greatly with regard to the outcome of that layer. If we now consider Figure~\ref{fig:gmw_over_time}, or more generally Section~\ref{sec:gmw}, we see that the Gaussian mean widths within the network of sets of correctly classified data points is much higher than the widths of sets of incorrectly classified data points. \textbf{This suggests that misclassifications are, at least partly, caused by a lack of correlations (of the data point) with singular vectors corresponding to big singular values, that are also not ``blocked'' by ReLU.} 
\subsection{Double-Layers}
In this subsection we present a simple tool to deploy the observational hypothesis discussed above in practice: double-layers. We define a (ReLU) \textbf{double-layer} as a layer of the form
\begin{equation}
\relu(WM\cdot+b),
\end{equation}
where $W\in\mathbb{R}^{m\times k}$, $M\in\mathbb{R}^{k\times n}$ and $k<\min(m,n)$. This structure enables one to enforce the layer's weight matrix to be low-rank, specifically to be of rank $\le k$.
Some properties of these layers are:
\begin{itemize}
\item One can create and control the size of a \textbf{linear bottleneck} between $W$ and $M.$
\item As soon as $k<\frac{mn}{m+n}$, which in the case $m=n$ reduces to $k<\frac{1}{2}n$, a double-layer has less parameters then a single-layer.
\item At least in some tasks, as we will see later, they seem to perform better than normal ReLU layer.
\item Since one can use their rank to adjust their expressibility, one can often permit $m=n$ in practice, which allows for a more effective usage of the variance argument for initialization as proposed by Glorot~\cite{glorot2010understanding} and He~\cite{he2015delving}. (Their argument struggles with the case $m\ne n$.)
\end{itemize}
Before we can use double-layers in practice we have to think about how we can initialize them to be in line with the widely used arguments of Glorot~\cite{glorot2010understanding} and He~\cite{he2015delving}. The argument results in the suggestion to use ReLU layers whose weight matrix $A\in\mathbb{R}^{m\times n}$ has entries that are randomly sampled with mean $0$ and a variance of $\frac{4}{m+n}.$ This is fulfilled for the following class of initializations that can be used for double-layers.
\begin{definition}
For $p\in\mathbb{Z}_+$ we define the \textbf{double-p-product} initialization of a ReLU (double-)layer of the form \[\relu(WM\cdot+b),\] where $W\in\mathbb{R}^{m\times k}$ and $M\in\mathbb{R}^{k\times n}$ (for normal layers this initialization can be used by setting $A=WM$) as follows: $b=0$ and all entries  $w_{i,j}$ and $m_{i,j}$ are products of $p$ i.\,i.\,d.~samples from $\mathcal{N}(0,\sqrt[\leftroot{-3}\uproot{3}2p]{\tfrac{4/k}{n+n}}).$
\end{definition}
Note that the singular value curve of $WM$ after initialization does, unlike for the usual Glorot initialization, not follow the Marchenko-Pastur distribution~\cite{marchenko1967distribution}, since the entries are not independent.
\subsection{Harmonic Pruning and Double-Layers in Practice}
We will now present a principled way to explore the use of linear bottlenecks using double-layers: A pruning algorithm that, inspired by our results, decreases the ranks of the weight matrices of a network and thereby also, at least for double-layers, the number of parameters. We call it \textbf{harmonic pruning}. It successively decreases the ranks of the weight matrices of an already trained MLP. The algorithm has the following steps:
\begin{enumerate}
\item For each layer calculate the change in accuracy
	of the whole network by setting the smallest non-zero singular value to zero.
Denote by $A$ the weight matrix that belongs to the layer that decreases the accuracy the least.
\item Calculate the SVD $A=U\Sigma V^T$ and let $\tilde\Sigma$ denote $\Sigma$ after the
	smallest non-zero singular value was set to zero.
	Also define $k:=\arg\min_i\left\{\tilde\Sigma_{i,i}:\tilde\Sigma_{i,i}\ne0\right\}$ as the index of the
	smallest non-zero singular value.
\item Reimplement the layer with a split weight matrix replacing $A$ by the rank constrained product $WM$.
	Here $W$ is given by the first $k$ columns of $U\sqrt{\tilde\Sigma}$ and $M$ is given by the first $k$ rows of $\sqrt{\tilde\Sigma}V^T$. $W$ and $M$ are implemented separately to enforce the upper rank (of $k$) during further steps. Overall we are cutting the rank down by one by setting the smallest non-zero singular value permanently to zero.
\item Retrain the network for some batches if some retraining criterion is if fulfilled, e.\,g.\ every $10$th iteration or if one of the layers has become very low rank.
\item If some stopping criterion is reached (e.\,g., the loss increases too much) stop, otherwise goto Step 2).
\end{enumerate}

We applied this algorithm with the retraining criterion to retrain whenever
\begin{itemize}
\item the accuracy drops by more than 0.5\% with regard to the initially trained network,
\item we are in a $10$th iteration, or
\item one of the layers already has rank less then $10$.
\end{itemize}
\begin{figure}[t]
\begin{subfigure}[t]{0.58\textwidth}
\includegraphics[width=0.8\linewidth]{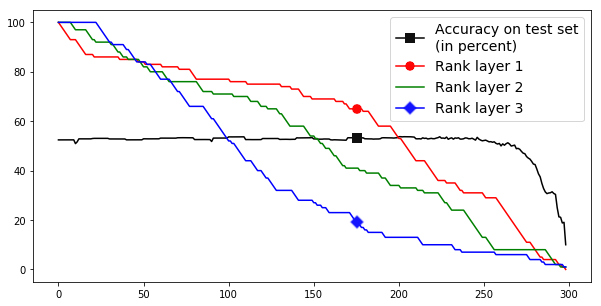}
\caption{CIFAR-10}
\end{subfigure}
\begin{subfigure}[t]{0.58\textwidth}
\includegraphics[width=0.8\linewidth]{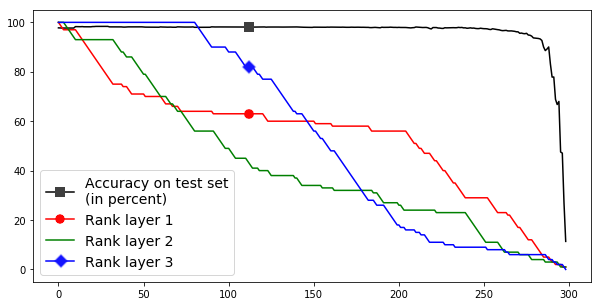}
\caption{MNIST}
\end{subfigure}
\caption{The plot displays the decay of the ranks over the pruning process and the changes in accuracy of the network on the training set. The horizontal axis represents the number of pruning steps, i.\,e., the iteration of the algorithm.}
\label{fig:overview_pruning}
\vspace{-02mm}
\end{figure}
For demonstrative purposes we set up an experiment that did not use any stopping criterion and pruned until the network had no parameters left. Every retraining used 50 batches (of size 1024). We pruned a network with three hidden ReLU double-layers of size $100$ and a softmax classification layer trained via cross-entropy-with-logits loss function and the Adam optimization algorithm. We used a batch size of 1024 and trained until the accuracy settled. The results of this procedure being run on two different classification tasks, on the CIFAR-10 and MNIST datasets, may be seen in Figure~\ref{fig:overview_pruning}. We trained for the classification task given by the CIFAR-10~\cite{krizhevsky2009learning} dataset and reached an accuracy of ca. 52\% over the test set -- in our tests MLPs of similar size and with normal ReLU layers and Glorot initializations reached similar or lower accuracies. We chose CIFAR-10 since we wanted to test a task that MLP is not overwhelmingly successful in achieving in order to reduce possible redundancies in the network. The resulting behavior of the network's ranks and accuracy over the test set during the pruning are displayed in Figure~\ref{fig:overview_pruning}(a). Figure~\ref{fig:overview_pruning}(b) shows the same setup, but based on MNIST, which yields a much higher success rate. The graph presents an impressively constant accuracy over the pruning process. We only see a drop in accuracy at the very end when the networks start to fail completely and the accuracy drops to the ``guessing baseline'' of 10\%. 
The network only drops permanently below 50\% mark after the iteration 268. Before that drop under 50\% the layer's weight matrices reached ranks of 18, 8, 6 respectively for the three ReLU layers of the network. This is for multiple reasons very interesting. One of the reasons is that the double-layer network with layer widths of 100, 100, 100 and ranks of 18, 8, 6 has only 61,206 parameters. A normal single-layer network with these layer widths has 328,510, i.e., more than 5 times as many parameters. When we trained a normal ReLU networks (i.e., no double-layers) with a comparable number of parameters (layer widths 20, 9, 7), we only reached about 40\% accuracy. But when we trained double-layer networks with layer widths of 100, 100, 100 and ranks 18, 8, 6, we again reached roughly 50\% accuracy. Unlike most other pruning methods this one does not decrease the number of active connections or neurons; rather, it synchronizes the neurons within a given layer using their common weight matrix to denoise all of them by removing the harmonics associated with less significant singular values. This is in some sense more related to pooling in CNNs than pruning, since pooling also denoises via dimensionality reduction by approximately removing the harmonics associated with high frequencies in the Fourier basis~\cite{rippel2015spectral}. The main differences here are that our method does not impose the dimensionality decrease via the Fourier basis but rather in layer-specific eigenbases and that we only decrease the dimensionality and not the number of output parameters. However, one could trivially decrease the number of output parameters by decreasing the layer width.

\section{Conclusion and Future Work}
\label{sec:conclusion}
In this work we explored the behavior and role of the singular values of ReLU layers and how they interact with the ReLU activation function. To do this we defined and explored ReLU singular values and the Gaussian mean width of operators. Our results do not only explain the success of linear bottlenecks, but also provide a principled way to use them in the form of double-layers. We see a multitude of open questions and future work directions.
For example, are there connections to the usual bottlenecks in ResNet blocks?
Can the measures presented in this paper be utilized
to determine fruitful matches of models and data in transfer learning? Also, can one detect overfitting using these methods? 

We think that both tools are theoretically interesting and further our understanding of the inner workings of neural networks. In practice ReLU singular values seem to be useful in the design of new architectures incorporating bottlenecks and the Gaussian mean width of operators promises to be a useful tool to analyze how well a given network is applicable to a given data set.

\section*{Acknowledgment}
The authors would like to thank Jens Behrmann, Jonathan von	Schroeder and Christian Etmann for their fruitful discussions and helpful comments on the paper.

S.\ Dittmer is supported by the Deutsche Forschungsgemeinschaft (DFG) within the framework of GRK 2224 / $\pi^3$: Parameter Identification - Analysis, Algorithms, Applications.

\bibliography{main} 
\bibliographystyle{ieeetr}
\vspace{-9mm}
\begin{IEEEbiographynophoto}{S\"oren Dittmer}
is a Ph.D.\ student and member of the Research Training Group $\pi^3$  at the Center for Industrial Mathematics (ZeTeM) at the University of Bremen, Germany. His current research interests include deep learning, inverse problems, harmonic analysis and signal processing.
\end{IEEEbiographynophoto}
\begin{IEEEbiographynophoto}{Emily J. King} has been an Assistant Professor
(Juniorprofessor) leading the research group Computational Data Analysis at the University of Bremen, Germany since 2014.  She earned a B.S. in Applied Mathematics and an M.S. in Mathematics from Texas A\&M
University, College Station, Texas, U.S.A., in 2003
and 2005, respectively. She then received a Ph.D.
in Mathematics from the University of Maryland,
College Park, Maryland, U.S.A., in 2009. From
2009 to 2011, she was an IRTA Postdoctoral Fellow
in the Laboratory for Integrative and Medical Biophysics at the National
Institutes of Health in Bethesda, Maryland, U.S.A. As an Alexander von Humboldt Postdoctoral Fellow, she worked from 2011 to 2013 at the University of Osnabr\"uck, the University of Bonn, and the Technical University of Berlin, all in Germany. Her research interests include algebraic and applied harmonic analysis,
signal and image processing, data analysis, and frame theory.
\end{IEEEbiographynophoto}
\begin{IEEEbiographynophoto}{Peter Maass}
is a Professor for Applied Mathematics and the Director of the Center for Industrial
Mathematics (ZeTeM) at University of Bremen, Germany, since 2009. He held positions as Assistant
Professor at Tufts University, Medford, MA, USA and Saarland University, Saarbr\"ucken, Germany,
before he was appointed as a Full Professor of Numerical Analysis at University of Potsdam,
Germany, in 1993. Peter Maass is an Adjunct Professor at Clemson University, SC, USA since 2010. He
holds several patents in the field of image processing. His current research interests include deep
learning, inverse problems and wavelet analysis with an emphasis on applications in medical imaging.
Peter Maass was awarded an honorary doctorate by the University of Saarland, Germany in 2018.
\end{IEEEbiographynophoto}

\appendices
\section{Additional Numerical Testing}
\label{appendix}
This appendix displays the same results as presented in the paper but for more networks and classification tasks. All networks are MLPs with three hidden ReLU layers of equal width and a softmax output layer trained via a cross-entropy-with-logits loss function and the Adam optimizer~\cite{kingma2014adam} solving a classification task given by some data set. They were all initialized with the Glorot initialization. The plots in each figure from left to right are the singular values of the linear map, the numerical bounds of the ReLU singular values (like in Figure~\ref{fig:singular_values_vs_relu_singular_values}), and the Gaussian mean width over the training (like in Figure~\ref{fig:gmw_over_time}). Solid lines are plots corresponding to correctly classified data, and dashed lines correspond to incorrectly classified data. We turned the Boston housing prices data set~\cite{harrison1978hedonic} into the classification task of classifying whether a house costs more than most houses.

\begin{figure*}
\begin{subfigure}{0.3\textwidth}
\includegraphics[width=\linewidth, trim={0 0 0 8mm}, clip]{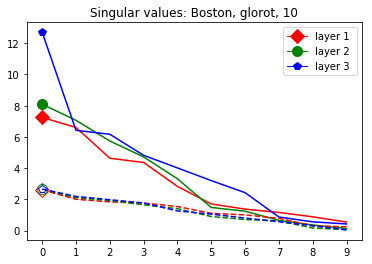}
\end{subfigure}\hspace*{\fill}
\begin{subfigure}{0.3\textwidth}
\includegraphics[width=\linewidth, trim={0 0 0 7mm}, clip]{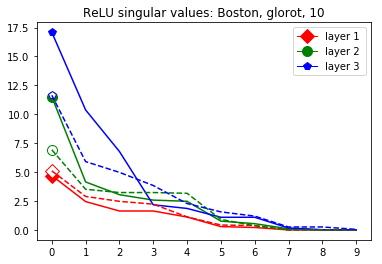}
\end{subfigure}\hspace*{\fill}
\begin{subfigure}{0.3\textwidth}
\includegraphics[width=\linewidth, trim={0 0 0 8mm}, clip]{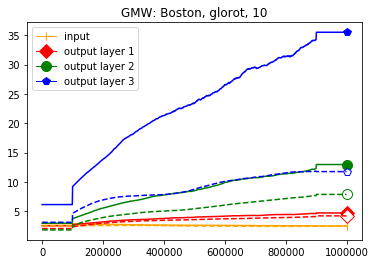}
\end{subfigure}
\caption{This network was trained on the Boston housing prices data set and the ReLU layers are of width 10.}
\end{figure*}

\begin{figure*}
\begin{subfigure}{0.3\textwidth}
\includegraphics[width=\linewidth, trim={0 0 0 8mm}, clip]{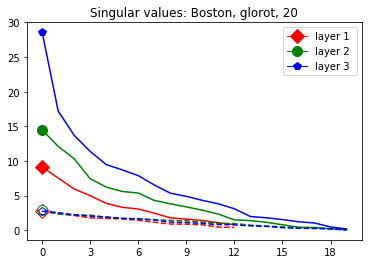}
\end{subfigure}\hspace*{\fill}
\begin{subfigure}{0.3\textwidth}
\includegraphics[width=\linewidth, trim={0 0 0 7mm}, clip]{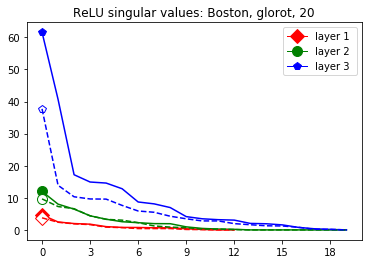}
\end{subfigure}\hspace*{\fill}
\begin{subfigure}{0.3\textwidth}
\includegraphics[width=\linewidth, trim={0 0 0 8mm}, clip]{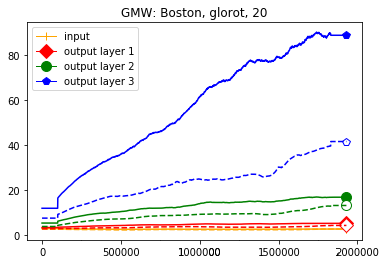}
\end{subfigure}
\caption{This network was trained on the Boston housing prices data set and the ReLU layers are of width 20.}
\end{figure*}

\begin{figure*}
\begin{subfigure}{0.3\textwidth}
\includegraphics[width=\linewidth, trim={0 0 0 8mm}, clip]{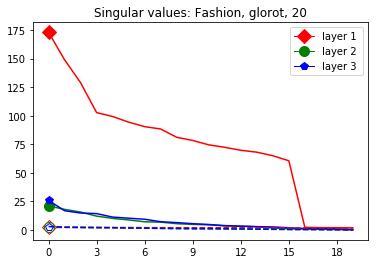}
\end{subfigure}\hspace*{\fill}
\begin{subfigure}{0.3\textwidth}
\includegraphics[width=\linewidth, trim={0 0 0 7mm}, clip]{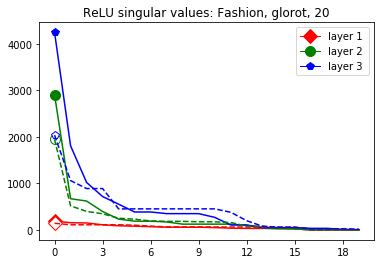}
\end{subfigure}\hspace*{\fill}
\begin{subfigure}{0.3\textwidth}
\includegraphics[width=\linewidth, trim={0 0 0 8mm}, clip]{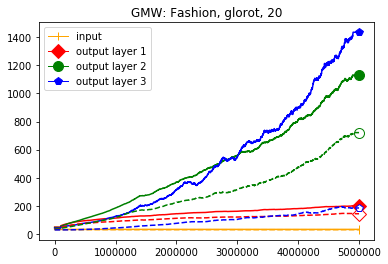}
\end{subfigure}
\caption{This network was trained on the Fashion MNIST data set and the ReLU layers are of width 20.}
\end{figure*}


\begin{figure*}
\begin{subfigure}{0.3\textwidth}
\includegraphics[width=\linewidth, trim={0 0 0 8mm}, clip]{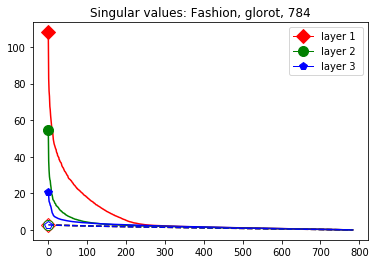}
\end{subfigure}\hspace*{\fill}
\begin{subfigure}{0.3\textwidth}
\includegraphics[width=\linewidth, trim={0 0 0 7mm}, clip]{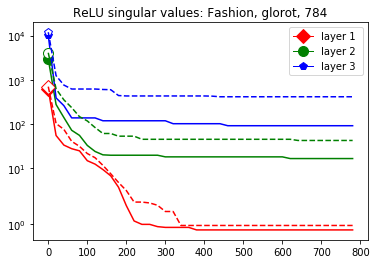}
\end{subfigure}\hspace*{\fill}
\begin{subfigure}{0.3\textwidth}
\includegraphics[width=\linewidth, trim={0 0 0 8mm}, clip]{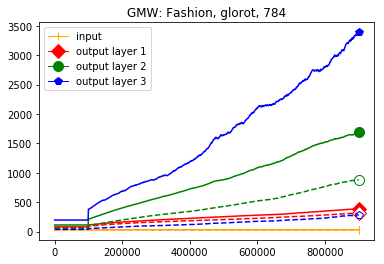}
\end{subfigure}
\caption{This network was trained on the Fashion MNIST data set and the ReLU layers are of width 784.}
\end{figure*}


\begin{figure*}
\begin{subfigure}{0.3\textwidth}
\includegraphics[width=\linewidth, trim={0 0 0 8mm}, clip]{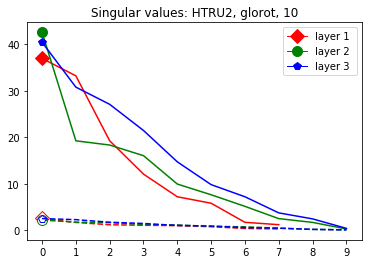}
\end{subfigure}\hspace*{\fill}
\begin{subfigure}{0.3\textwidth}
\includegraphics[width=\linewidth, trim={0 0 0 7mm}, clip]{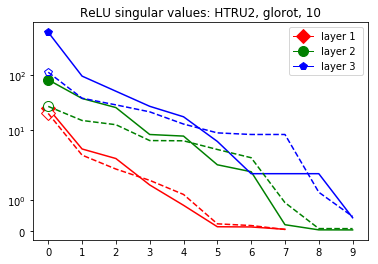}
\end{subfigure}\hspace*{\fill}
\begin{subfigure}{0.3\textwidth}
\includegraphics[width=\linewidth, trim={0 0 0 8mm}, clip]{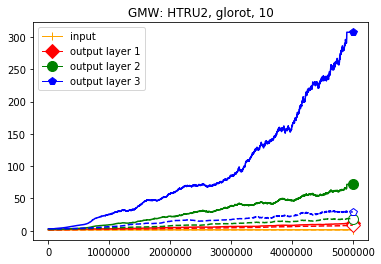}
\end{subfigure}
\caption{This network was trained on the HTRU2 data set and the ReLU layers are of width 10.}
\end{figure*}


\begin{figure*}
\begin{subfigure}{0.3\textwidth}
\includegraphics[width=\linewidth, trim={0 0 0 8mm}, clip]{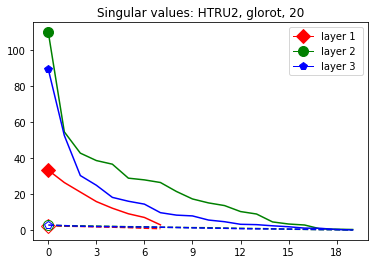}
\end{subfigure}\hspace*{\fill}
\begin{subfigure}{0.3\textwidth}
\includegraphics[width=\linewidth, trim={0 0 0 7mm}, clip]{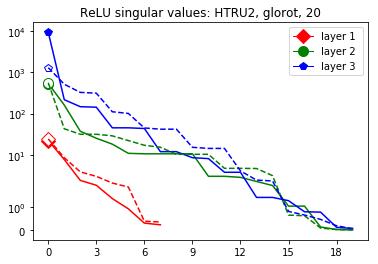}
\end{subfigure}\hspace*{\fill}
\begin{subfigure}{0.3\textwidth}
\includegraphics[width=\linewidth, trim={0 0 0 8mm}, clip]{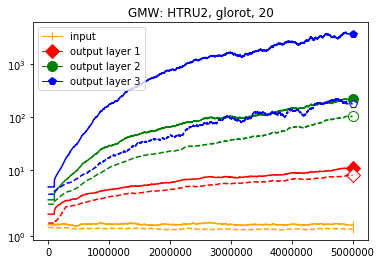}
\end{subfigure}
\caption{This network was trained on the HTRU2 data set and the ReLU layers are of width 20.}
\end{figure*}

\begin{figure*}
\begin{subfigure}{0.3\textwidth}
\includegraphics[width=\linewidth, trim={0 0 0 8mm}, clip]{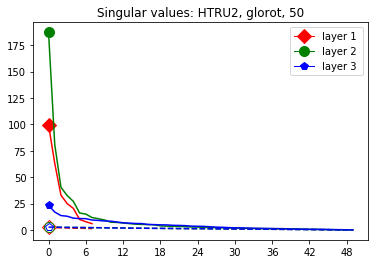}
\end{subfigure}\hspace*{\fill}
\begin{subfigure}{0.3\textwidth}
\includegraphics[width=\linewidth, trim={0 0 0 7mm}, clip]{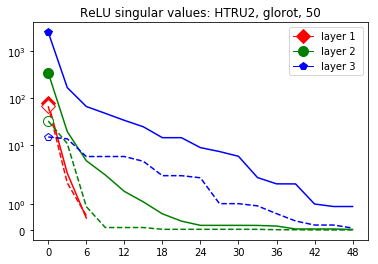}
\end{subfigure}\hspace*{\fill}
\begin{subfigure}{0.3\textwidth}
\includegraphics[width=\linewidth, trim={0 0 0 8mm}, clip]{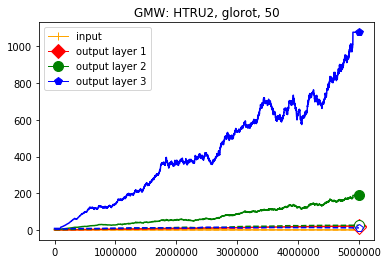}
\end{subfigure}
\caption{This network was trained on the HTRU2 data set and the ReLU layers are of width 50.}
\end{figure*}

\begin{figure*}
\begin{subfigure}{0.3\textwidth}
\includegraphics[width=\linewidth, trim={0 0 0 8mm}, clip]{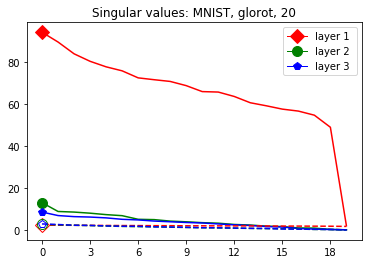}
\end{subfigure}\hspace*{\fill}
\begin{subfigure}{0.3\textwidth}
\includegraphics[width=\linewidth, trim={0 0 0 7mm}, clip]{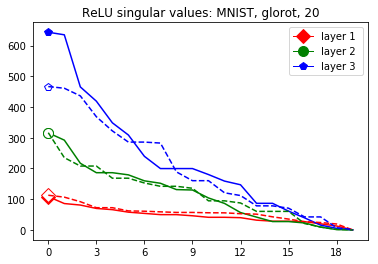}
\end{subfigure}\hspace*{\fill}
\begin{subfigure}{0.3\textwidth}
\includegraphics[width=\linewidth, trim={0 0 0 8mm}, clip]{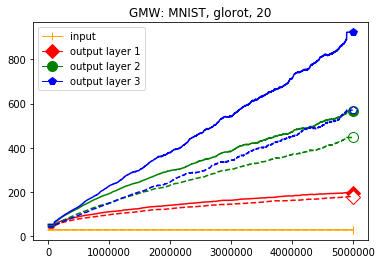}
\end{subfigure}
\caption{This network was trained on the MNIST data set and the ReLU layers are of width 20.}
\end{figure*}

\begin{figure*}
\begin{subfigure}{0.3\textwidth}
\includegraphics[width=\linewidth, trim={0 0 0 8mm}, clip]{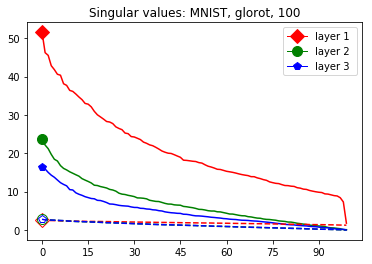}
\end{subfigure}\hspace*{\fill}
\begin{subfigure}{0.3\textwidth}
\includegraphics[width=\linewidth, trim={0 0 0 7mm}, clip]{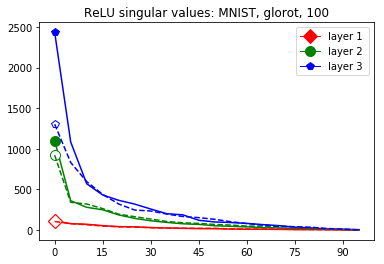}
\end{subfigure}\hspace*{\fill}
\begin{subfigure}{0.3\textwidth}
\includegraphics[width=\linewidth, trim={0 0 0 8mm}, clip]{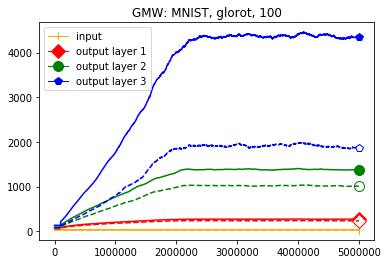}
\end{subfigure}
\caption{This network was trained on the MNIST data set and the ReLU layers are of width 100.}
\end{figure*}

\begin{figure*}
\begin{subfigure}{0.3\textwidth}
\includegraphics[width=\linewidth, trim={0 0 0 8mm}, clip]{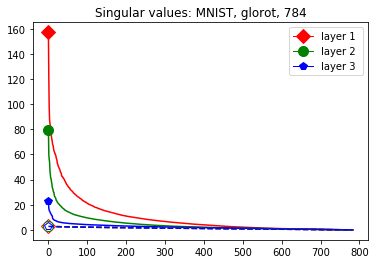}
\end{subfigure}\hspace*{\fill}
\begin{subfigure}{0.3\textwidth}
\includegraphics[width=\linewidth, trim={0 0 0 7mm}, clip]{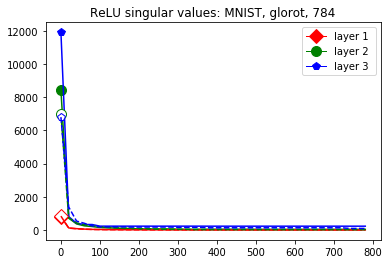}
\end{subfigure}\hspace*{\fill}
\begin{subfigure}{0.3\textwidth}
\includegraphics[width=\linewidth, trim={0 0 0 8mm}, clip]{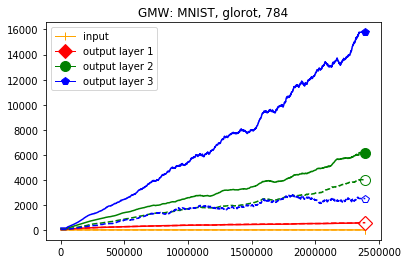}
\end{subfigure}
\caption{This network was trained on the MNIST data set and the ReLU layers are of width 784.}
\end{figure*}

\end{document}